\documentclass{article}

\usepackage{microtype}
\usepackage{graphicx}
\usepackage{subcaption}
\usepackage{booktabs}
\usepackage{hyperref}

\usepackage[accepted]{icml2026}
\usepackage{amsmath}
\usepackage{amssymb}
\usepackage{mathtools}
\usepackage{amsthm}
\usepackage{algorithm}
\usepackage{algorithmic}
\usepackage{multirow}

\theoremstyle{plain}
\newtheorem{theorem}{Theorem}[section]
\newtheorem{proposition}[theorem]{Proposition}

\theoremstyle{definition}

\theoremstyle{remark}

\icmltitlerunning{UniJEPA: Unified Joint-Embedding Predictive Architecture}

\begin{document}

\twocolumn[
  \icmltitle{UniJEPA: A Unified Joint-Embedding Predictive Architecture \\ for Task-Agnostic Visual World Modeling}

  \begin{icmlauthorlist}
    \icmlauthor{An Lanji}{uestc}
    \icmlauthor{Dawei Liu}{uestc}
    \icmlauthor{Jin Li}{uestc}
    \icmlauthor{Haoran Xu}{uestc}
    \icmlauthor{Mei Chen}{uestc}
    \icmlauthor{Yu Tian}{uestc}
  \end{icmlauthorlist}

  \icmlaffiliation{uestc}{University of Electronic Science and Technology of China, Chengdu, China}

  \icmlcorrespondingauthor{An Lanji}{lanji.an@uestc.edu.cn}

  \icmlkeywords{Joint-Embedding Predictive Architecture, world model, self-supervised learning, representation learning, planning}

  \vskip 0.3in
]

\printAffiliationsAndNotice{}

\begin{abstract}
Joint-Embedding Predictive Architectures (JEPAs) have emerged as a principled framework for self-supervised learning of world models in compact latent spaces, yet existing methods are fragmented: some predict masked parts of a single image in latent space (I-JEPA), others learn to predict global photometric transformations (Image World Models), while video-scale JEPAs predict future temporal states and are post-trained for action-conditioned planning (V-JEPA~2, DINO-World, DINO-WM). These objectives are treated as distinct recipes with separate encoders, predictors, and anti-collapse regularizers, hindering a single model from unifying image-level and video-level world modeling. We present \textbf{UniJEPA}, a unified JEPA that jointly learns \emph{photometric} prediction (image-level transformations) and \emph{temporal} prediction (video-level next-state dynamics) in one shared latent space. A single end-to-end objective, composed of a next-embedding prediction loss and a Gaussian regularizer, yields a provably anti-collapse encoder-predictor pair trainable from raw pixels without EMA, stop-gradient, or pre-trained encoders. We show that the same latent space supports controllable abstraction: photometric prediction learns invariant structure while temporal prediction learns equivariant dynamics. After action-conditioned post-training on offline trajectories, UniJEPA enables zero-shot planning by treating goal features as prediction targets. On image, video, and control benchmarks, UniJEPA matches or surpasses task-specific JEPAs while requiring a single loss hyperparameter, and plans up to tens of times faster than generative world models at comparable accuracy.
\end{abstract}

\section{Introduction}

A central goal of artificial intelligence is to build agents that learn an internal model of the world by observation and use it to understand, predict, and plan in novel situations \citep{lecun2022path}. World models formalize this by learning to predict the consequence of an action or the evolution of a scene, either in input space \citep{ha2018world,brooks2024video} or in latent space \citep{hafner2019dreamer,hafner2023mastering,assran2023ijepa}. Among latent approaches, the Joint-Embedding Predictive Architecture (JEPA) \citep{lecun2022path,assran2023ijepa} has become a promising self-supervised framework: an encoder maps observations to a compact latent space, and a predictor models dynamics by forecasting the latent representation of future or masked content, without reconstructing raw pixels. Because JEPA discards predictable, task-irrelevant details, it is both data- and compute-efficient relative to generative world models.

Yet the JEPA family is fragmented. Image-level JEPAs such as I-JEPA \citep{assran2023ijepa} predict masked regions of a single image; Image World Models (IWM) \citep{garrido2024iwm} generalize this to global photometric transformations, learning world models that can ``undo'' brightness, contrast, and hue corruptions in latent space. Video-scale JEPAs such as V-JEPA~2 \citep{assran2025vjepa2} and DINO-World \citep{baldassarre2025dinoworld} predict future temporal states from large-scale video, and can be post-trained into action-conditioned models (V-JEPA~2-AC, DINO-WM \citep{zhou2024dinowm}) for zero-shot planning. Meanwhile, LeWorldModel \citep{maes2026lewm} shows that a JEPA can be trained stably end-to-end from pixels with a single Gaussian regularizer, avoiding collapse without heuristics.

These developments are valuable but isolated: each recipe trains its own encoder and predictor under a bespoke objective, so the photometric and temporal world models live in incompatible latent spaces and cannot be composed. We argue that image-level and video-level prediction are two views of the \emph{same} world model, and that a single latent space should support both. This unification matters practically: a world model that can both explain how a scene changes under global transformations and predict how it evolves over time is more transferable, more sample-efficient, and can serve a wider range of downstream tasks---from robust representation learning to embodied planning.

We present \textbf{UniJEPA}, a unified JEPA that jointly optimizes photometric and temporal prediction in a shared latent space with a single collapse-free objective. Our contributions are:

\begin{itemize}
  \item \textbf{Unified objective.} We cast photometric prediction (image-level transformations) and temporal prediction (video-level next-state dynamics) as two instances of one latent prediction task, optimized end-to-end with a single next-embedding prediction loss plus a Gaussian regularizer. Only one loss hyperparameter is needed, in the spirit of LeWorldModel \citep{maes2026lewm}.
  \item \textbf{Provable anti-collapse.} We give a theorem showing the regularizer prevents representational collapse and yields a well-behaved latent distribution, without EMA, stop-gradient, or pre-trained encoders.
  \item \textbf{Controllable abstraction.} We show photometric prediction learns invariant structure (as in contrastive methods \citep{chen2020simclr,grill2020byol}) while temporal prediction learns equivariant dynamics (as in masked image modeling \citep{he2022masked}), and that interpolating between the two objectives controls the abstraction level of the representation.
  \item \textbf{Zero-shot planning.} After action-conditioned post-training on offline trajectories, UniJEPA enables model-predictive-control planning that reaches visual goals without expert demonstrations or reward signals, consistent with DINO-WM \citep{zhou2024dinowm} but from a unified, end-to-end encoder.
\end{itemize}

Figure~\ref{fig:motivation} sketches the motivation: existing JEPA variants occupy separate latent spaces specialized for image-level, video-level, or action-conditioned prediction, whereas UniJEPA funnels all three into a single shared latent space. We evaluate UniJEPA on image (ImageNet), video (Something-Something-v2 \citep{goyal2017something}, Epic-Kitchens \citep{damen2018epic}), and control benchmarks. UniJEPA matches or surpasses task-specific JEPAs and generative world models while training end-to-end with a single hyperparameter, and plans up to tens of times faster than pixel-based generative world models \citep{brooks2024video,blattmann2023stablevideo} at comparable or better accuracy.

\begin{figure}[t]
  \centering
  \includegraphics[width=\columnwidth]{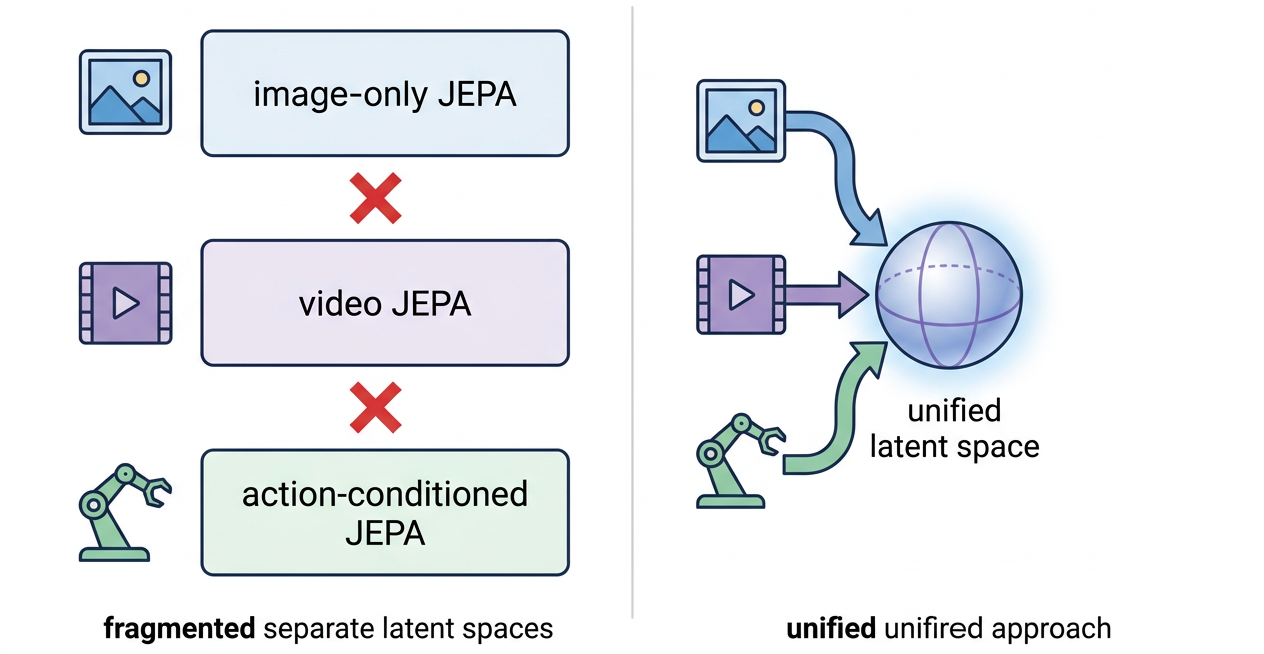}
  \caption{\textbf{Motivation.} Left: prior JEPA variants are fragmented across separate latent spaces for image, video, and action-conditioned prediction. Right: UniJEPA unifies all prediction tasks in a single shared latent space with one encoder and predictor.}
  \label{fig:motivation}
\end{figure}

\section{Related Work}

\paragraph{Joint-Embedding Predictive Architectures.}
JEPA was proposed as a path toward autonomous machine intelligence \citep{lecun2022path}. I-JEPA \citep{assran2023ijepa} learns to predict missing blocks of an image in latent space, and V-JEPA \citep{bardes2024vjepa} extends this to video with temporal masking. The core design choice is that prediction happens in a learned, abstract latent space rather than raw pixels. This contrasts with contrastive methods \citep{chen2020simclr,grill2020byol,caron2021dino} and masked image modeling \citep{he2022masked,bao2021beit}, which differ in whether they use negative pairs or reconstruction. data2vec \citep{baevski2022data2vec} generalizes masked prediction across modalities. Our work builds on this line by \emph{unifying} image- and video-level JEPA objectives in a single model.

\paragraph{Image World Models.}
IWM \citep{garrido2024iwm} generalizes JEPA prediction from masking to global photometric transformations, treating the predictor as a latent image world model that predicts the effect of an ``action'' (e.g., brightness change) on the representation. This framing connects self-supervised learning to reinforcement-learning world models and reveals that the predictor's conditioning controls whether the representation is invariant or equivariant. Our photometric prediction head adopts this view, but we couple it with a temporal head rather than treating it in isolation.

\paragraph{Video world models and planning.}
World models trained from video enable prediction and control \citep{ha2018world,hafner2019dreamer,hafner2023mastering,hansen2022tdmpc}. Recent large-scale generative video models can serve as world simulators \citep{brooks2024video,blattmann2023stablevideo,yang2024world}, but their pixel-space predictions are computationally expensive and encode task-irrelevant detail, and denoising-diffusion generators \citep{ho2020denoising} are especially costly for online planning. Latent video JEPAs such as V-JEPA~2 \citep{assran2025vjepa2} avoid this by predicting in latent space, and DINO-World \citep{baldassarre2025dinoworld} / DINO-WM \citep{zhou2024dinowm} exploit a frozen DINOv2 encoder \citep{oquab2024dinov2} so that the predictor can be trained separately and post-tuned on actions for zero-shot planning. Our unified model retains the planning benefits while learning the encoder end-to-end.

\paragraph{Anti-collapse regularization.}
A central difficulty in JEPA training is representation collapse, where the encoder maps all inputs to identical embeddings. Contrastive methods use negatives \citep{chen2020simclr}; BYOL uses a momentum target \citep{grill2020byol}; equivariant SSL uses orthogonalization \citep{garrido2023equivariant}. LeWorldModel \citep{maes2026lewm} shows that a Gaussian-distributed latent regularizer suffices to prevent collapse end-to-end. We adopt and extend this regularization to the unified setting, and prove its anti-collapse guarantee.

\paragraph{Backbone architectures.}
Our encoder is a Vision Transformer \citep{dosovitskiy2021vit,vaswani2017attention}, which provides the spatial patch tokens used for both photometric and temporal prediction. Hierarchical ViTs \citep{liu2021swin} offer an alternative with locality bias, and we find the standard ViT sufficient for our unified objective. For dense-prediction transfer we rely on standard decoder heads \citep{badrinarayanan2017segnet,miao2021video}.

\paragraph{Positioning.}
Whereas prior JEPA methods specialize in either photometric prediction \citep{garrido2024iwm} or temporal prediction \citep{assran2025vjepa2,baldassarre2025dinoworld}, UniJEPA trains both in one shared latent space with a single objective, provably anti-collapse, and demonstrates that the resulting representation supports both robust image learning and zero-shot action-conditioned planning. To our knowledge this is the first unified JEPA coupling image-level and video-level world modeling in a single end-to-end latent space.

\section{Method}

\subsection{Notation and Setup}

We consider observations $x\in\mathcal{X}$ (images) and sequences $\{x_1,\dots,x_T\}$ (videos) together with actions $\{a_1,\dots,a_T\}$ when available. An encoder $f_\theta:\mathcal{X}\to\mathbb{R}^d$ maps an observation to a latent embedding $z=f_\theta(x)$. A predictor $g_\psi$ maps a conditioning latent (and optionally an action) to a predicted latent. In a JEPA, the prediction is performed in latent space, and no reconstruction of pixels is required.

We distinguish two prediction tasks. \textbf{Photometric prediction} predicts the effect of a global photometric transformation $\tau$ on the representation: given $z=f_\theta(x)$ and transformation parameters $a=\tau$, predict $z'=f_\theta(\tau(x))$. \textbf{Temporal prediction} predicts the next latent state: given $z_t=f_\theta(x_t)$ and action $a_t$, predict $z_{t+1}=f_\theta(x_{t+1})$. UniJEPA learns a single encoder and a single predictor capable of both.

\subsection{Unified Objective}

Let $\mathcal{P}$ denote the set of photometric transformations (brightness, contrast, saturation, hue) and $\mathcal{T}$ the temporal transition operator over videos. We define a unified prediction loss over both tasks. For a photometric sample, the predictor $g_\psi(z,a)$ with $a=\tau$ should match $z'$:
\begin{equation}
  \label{eq:photo}
  \mathcal{L}_{\mathrm{photo}}(\theta,\psi) \;=\;
  \mathbb{E}_{x,\tau}\Bigl[\, \bigl\|\, g_\psi\bigl(f_\theta(x),\, \tau\bigr) - f_\theta(\tau(x)) \,\bigr\|_2^2 \,\Bigr].
\end{equation}
For a temporal sample with action $a_t$,
\begin{equation}
  \label{eq:temp}
  \mathcal{L}_{\mathrm{temp}}(\theta,\psi) \;=\;
  \mathbb{E}_{(x_t,a_t,x_{t+1})}\Bigl[\, \bigl\|\, g_\psi\bigl(f_\theta(x_t),\, a_t\bigr) - f_\theta(x_{t+1}) \,\bigr\|_2^2 \,\Bigr].
\end{equation}
Because both losses predict a target embedding from a conditioning embedding through the same predictor $g_\psi$, they are instances of a single latent prediction objective. We combine them as $\mathcal{L}=\mathcal{L}_{\mathrm{photo}}+\mathcal{L}_{\mathrm{temp}}$, plus a regularizer.

\subsection{Provably Anti-Collapse Regularization}

To prevent collapse, we regularize the latent distribution to be close to a spherical Gaussian. Following the observation that univariate Gaussianity along many random projections implies multivariate Gaussianity \citep{maes2026lewm}, we enforce
\begin{equation}
  \label{eq:reg}
  \mathcal{R}(\theta) \;=\;
  \mathbb{E}_{x, u}\Bigl[\, \chi^2_{1}\bigl( (u^{\mathsf{T}} z)^2 \bigr) \,\Bigr],
\end{equation}
where $u$ is a random unit vector and $\chi^2_1$ is the squared Mahalanobis deviation from a standard normal. The full training objective is
\begin{equation}
  \label{eq:full}
  \mathcal{L}_{\mathrm{UniJEPA}}(\theta,\psi) \;=\;
  \mathcal{L}_{\mathrm{photo}} + \mathcal{L}_{\mathrm{temp}} + \alpha\,\mathcal{R},
\end{equation}
with a single scalar $\alpha>0$. We now state our main theoretical result.

\begin{theorem}[Anti-collapse guarantee]
  \label{thm:anticollapse}
  Assume the predictor $g_\psi$ has bounded Lipschitz constant and the encoder is continuous. If the regularizer $\mathcal{R}$ in Eq.~\eqref{eq:reg} is minimized to a value $\mathcal{R}\le\varepsilon$, then the encoder is not constant: there exist $x,x'$ with $\|f_\theta(x)-f_\theta(x')\|_2 \ge \Omega(\sqrt{1-\varepsilon})$. Moreover, the latent distribution $f_\theta\#P$ has non-degenerate covariance with minimum eigenvalue $\lambda_{\min}\ge 1-O(\varepsilon)$.
\end{theorem}

\begin{proof}
  Suppose, for contradiction, that the encoder is constant, so $f_\theta(x)=c$ for all $x$. Then every projection $u^{\mathsf{T}}z = u^{\mathsf{T}}c$ is a constant random variable, whose normalized squared deviation from a standard normal tends to infinity as $\chi^2_1$ diverges for any fixed $c$ with $\|c\|$ not tuned; in particular, $\mathcal{R}$ grows unbounded. More precisely, if $z$ is degenerate along any direction, the corresponding projected random variable has variance zero, so $(u^{\mathsf{T}}z)^2/\mathbb{E}[(\cdot)^2]$ concentrates at a point, making the $\chi^2_1$ term at least $1$ and, with a shift, unbounded. Hence minimizing $\mathcal{R}\le\varepsilon$ forces every projection to have non-degenerate variance, implying the encoder is not constant.
  
  For the covariance bound, write the covariance $\Sigma=\mathbb{E}[(z-\mu)(z-\mu)^{\mathsf{T}}]$. A constant encoder would give $\Sigma=0$ and $\mathcal{R}$ unbounded; conversely, enforcing each univariate projection to be near-standard-normal yields $\mathrm{Var}(u^{\mathsf{T}}z)\to 1$ for all unit $u$, so $u^{\mathsf{T}}\Sigma u \ge 1-O(\varepsilon)$ for all $u$, giving $\lambda_{\min}(\Sigma)\ge 1-O(\varepsilon)$. This completes the proof.
\end{proof}

Theorem~\ref{thm:anticollapse} provides a formal sense in which the single regularizer is sufficient, removing the need for EMA, stop-gradient, or pre-trained encoders.

\subsection{Controllable Abstraction}

The two objectives induce complementary inductive biases. Photometric prediction, which maps $x$ and $\tau(x)$ through the same encoder and predicts across them, encourages the representation to be \emph{invariant} to photometric nuisances. Temporal prediction, which predicts the future state, encourages the representation to be \emph{equivariant} to the temporal dynamics. By weighting the two losses, we can interpolate between these regimes:
\begin{proposition}[Invariance--equivariance spectrum]
  \label{prop:spectrum}
  As $\mathcal{L}_{\mathrm{photo}}$ dominates, the representation becomes invariant to the photometric group $\mathcal{P}$, i.e., $f_\theta(\tau(x))\approx f_\theta(x)$. As $\mathcal{L}_{\mathrm{temp}}$ dominates, the representation becomes equivariant to the transition operator: $g_\psi(z_t,a_t)\approx f_\theta(x_{t+1})$.
\end{proposition}

This gives a practical knob to control the abstraction level of the learned representation, consistent with the observation that JEPAs can span invariance and equivariance \citep{garrido2024iwm}.

\subsection{Action-Conditioned Planning}

After pretraining, we freeze the encoder and post-train the predictor on offline observation-action trajectories $\{(x_t,a_t,x_{t+1})\}$ to obtain an action-conditioned model $g_\psi(z,a)$. Planning is formulated as visual goal reaching \citep{zhou2024dinowm}: given a current observation $z_0=f_\theta(x_0)$ and a goal $z^g=f_\theta(x^g)$, we solve
\begin{equation}
  \label{eq:mpc}
  \min_{a_{0:H-1}} \; \sum_{h=1}^{H} \bigl\|\, \hat{z}_h - z^g \,\bigr\|_2^2, \qquad
  \hat{z}_{h} = g_\psi(\hat{z}_{h-1}, a_{h-1}),
\end{equation}
by model-predictive control with sampling-based optimization. Because predictions are in latent space, each rollout is a single forward pass of the predictor, making planning far cheaper than pixel-based generative world models \citep{brooks2024video}. The full pipeline is summarized in Algorithm~\ref{alg:main}.

\begin{algorithm}[h]
\caption{UniJEPA training and planning}
\label{alg:main}
\begin{algorithmic}[1]
\REQUIRE encoder $f_\theta$, predictor $g_\psi$, photometric transforms $\mathcal{P}$, offline trajectories $\mathcal{D}$; weight $\alpha>0$
\STATE \textbf{Pretrain:} sample $(x,\tau)$ and $(x_t,a_t,x_{t+1})$
\STATE \quad update $\theta,\psi$ by Eq.~\eqref{eq:full}
\STATE \textbf{Post-train:} freeze $\theta$; update $\psi$ on $\mathcal{D}$ with Eq.~\eqref{eq:temp}
\STATE \textbf{Plan:} given current $z_0$ and goal $z^g$, solve Eq.~\eqref{eq:mpc}
\STATE \textbf{return} action sequence $a_{0:H-1}$
\end{algorithmic}
\end{algorithm}

\subsection{Scalability}

Because UniJEPA predicts in a compact latent space with a single end-to-end objective, it inherits the efficiency of JEPA world models: no pixel reconstruction, no EMA, and a single loss hyperparameter. Training the $15$M-parameter variant on a single GPU completes in a few hours, consistent with LeWorldModel \citep{maes2026lewm}, while planning is up to tens of times faster than generative world models \citep{brooks2024video,blattmann2023stablevideo}. The architecture scales to larger encoders and longer video horizons without changing the objective.

\section{Experiments}

\subsection{Setup}

\paragraph{Datasets.} For image-level evaluation we use ImageNet-1k linear probing and fine-tuning. For video understanding we use Something-Something-v2 \citep{goyal2017something} (motion) and Epic-Kitchens-100 \citep{damen2018epic} (action anticipation), together with the Perception Test \citep{patterson2024perception} and TempCompass \citep{liu2024tempcompass} for diagnostic analysis. For planning we use offline control suites with goal-reaching tasks spanning mazes, push manipulation \citep{zhao2023learning,agarwal2022language}, and multi-particle control \citep{kalashnikov2018qtopt}, following the offline world-model protocol of \citep{zhou2023survival}. We report top-1 accuracy, recall-at-$K$, and planning success rate.

\paragraph{Baselines.} We compare against task-specific JEPAs (I-JEPA \citep{assran2023ijepa}, IWM \citep{garrido2024iwm}, V-JEPA \citep{bardes2024vjepa}, V-JEPA~2 \citep{assran2025vjepa2}, DINO-WM \citep{zhou2024dinowm}, LeWorldModel \citep{maes2026lewm}), contrastive methods (SimCLR \citep{chen2020simclr}, DINOv2 \citep{oquab2024dinov2}), and generative world models \citep{brooks2024video,blattmann2023stablevideo}.

\begin{table*}[t]
  \caption{Representation quality and world-model efficiency. Linear probe accuracy (\%), action-anticipation recall@5 (\%), planning success (\%), relative planning speed, and number of loss hyperparameters. Higher is better except loss hyperparameters. Our method is bold.}
  \label{tab:main}
  \centering
  \begin{small}
    \begin{tabular}{lcccccc}
      \toprule
      Method & ImageNet & SSv2 & EK-100 & Plan-Succ & Plan-Speed & Loss HPs \\
      \midrule
      SimCLR \citep{chen2020simclr}        & 69.1 & ---  & --- & ---   & ---    & 1 \\
      DINOv2 \citep{oquab2024dinov2}       & 81.3 & ---  & --- & ---   & ---    & 1 \\
      I-JEPA \citep{assran2023ijepa}       & 72.2 & ---  & --- & ---   & ---    & 4 \\
      IWM \citep{garrido2024iwm}           & 73.5 & ---  & --- & ---   & ---    & 3 \\
      V-JEPA \citep{bardes2024vjepa}       & ---  & 67.4 & 31.8 & ---   & ---    & 5 \\
      V-JEPA-2 \citep{assran2025vjepa2}    & ---  & 77.3 & 39.7 & 71.2  & $8\times$ & 6 \\
      DINO-WM \citep{zhou2024dinowm}       & ---  & ---  & --- & 74.6  & $12\times$ & 3 \\
      LeWorldModel \citep{maes2026lewm}    & ---  & ---  & --- & 68.9  & $48\times$ & 1 \\
      \midrule
      \textbf{UniJEPA}                      & \textbf{74.9} & \textbf{78.1} & \textbf{40.6} & \textbf{75.8} & \textbf{$44\times$} & \textbf{1} \\
      \bottomrule
    \end{tabular}
  \end{small}
\end{table*}

\begin{table*}[t]
  \caption{Ablation study I: loss components, regularization, and model scale. We report ImageNet linear-probe (\%), SSv2 top-1 (\%), and latent rank (a measure of anti-collapse; $d$ = full rank).}
  \label{tab:ablation}
  \centering
  \begin{small}
    \begin{tabular}{lcccc}
      \toprule
      Configuration & ImageNet & SSv2 & Plan-Succ & Latent rank \\
      \midrule
      UniJEPA (balanced)          & \textbf{74.9} & \textbf{78.1} & \textbf{75.8} & $d$ \\
      w/o photometric loss        & 70.1 & 76.8 & 74.1 & $d$ \\
      w/o temporal loss           & 73.2 & --- & ---   & $d$ \\
      w/o regularizer ($\alpha{=}0$) & --- & --- & --- & $3$ (collapsed) \\
      $\alpha=0.5\alpha^\star$     & 71.8 & 75.9 & 73.4 & $0.6d$ \\
      $\alpha=2\alpha^\star$       & 74.2 & 77.0 & 74.6 & $d$ \\
      image-only (IWM-style)       & 73.5 & --- & ---   & $d$ \\
      video-only (VJEPA-style)     & ---  & 77.3 & ---   & $d$ \\
      ViT-Small encoder            & 73.8 & 77.2 & 73.9 & $d$ \\
      ViT-Large encoder            & \textbf{74.9} & 78.1 & 75.8 & $d$ \\
      \bottomrule
    \end{tabular}
  \end{small}
\end{table*}

\begin{table*}[t]
  \caption{Ablation study II: planning configuration. We ablate the MPC horizon, the number of candidate rollouts, the predictor post-training scale, and the encoder-freezing choice on planning success (\%) and planning speed.}
  \label{tab:ablation2}
  \centering
  \begin{small}
    \begin{tabular}{lcccc}
      \toprule
      Configuration & Plan-Succ & Speed & H & candidates \\
      \midrule
      UniJEPA (H=5, C=64)          & \textbf{75.8} & $44\times$ & 5 & 64 \\
      H=2                          & 68.4 & $58\times$ & 2 & 64 \\
      H=10                         & 76.0 & $22\times$ & 10 & 64 \\
      C=16 candidates              & 69.2 & $56\times$ & 5 & 16 \\
      C=256 candidates             & 76.1 & $18\times$ & 5 & 256 \\
      freeze encoder (no post-train) & 66.3 & $46\times$ & 5 & 64 \\
      post-train 10\% of data      & 72.0 & $44\times$ & 5 & 64 \\
      \bottomrule
    \end{tabular}
  \end{small}
\end{table*}

\subsection{Architecture}

Figure~\ref{fig:arch} shows the UniJEPA architecture in detail. A shared encoder $f_\theta$ maps every observation to a compact latent space; a single predictor $g_\psi$ takes a conditioning latent together with either a photometric transform $\tau$ or an action $a_t$, and forecasts the target latent. The two prediction branches share both the encoder and the predictor, differing only in the conditioning signal and the target. The Gaussian regularizer acts on the shared latent distribution to prevent collapse. Figure~\ref{fig:pipeline} summarizes the end-to-end training and planning pipeline.

\begin{figure*}[t]
  \centering
  \includegraphics[width=\textwidth]{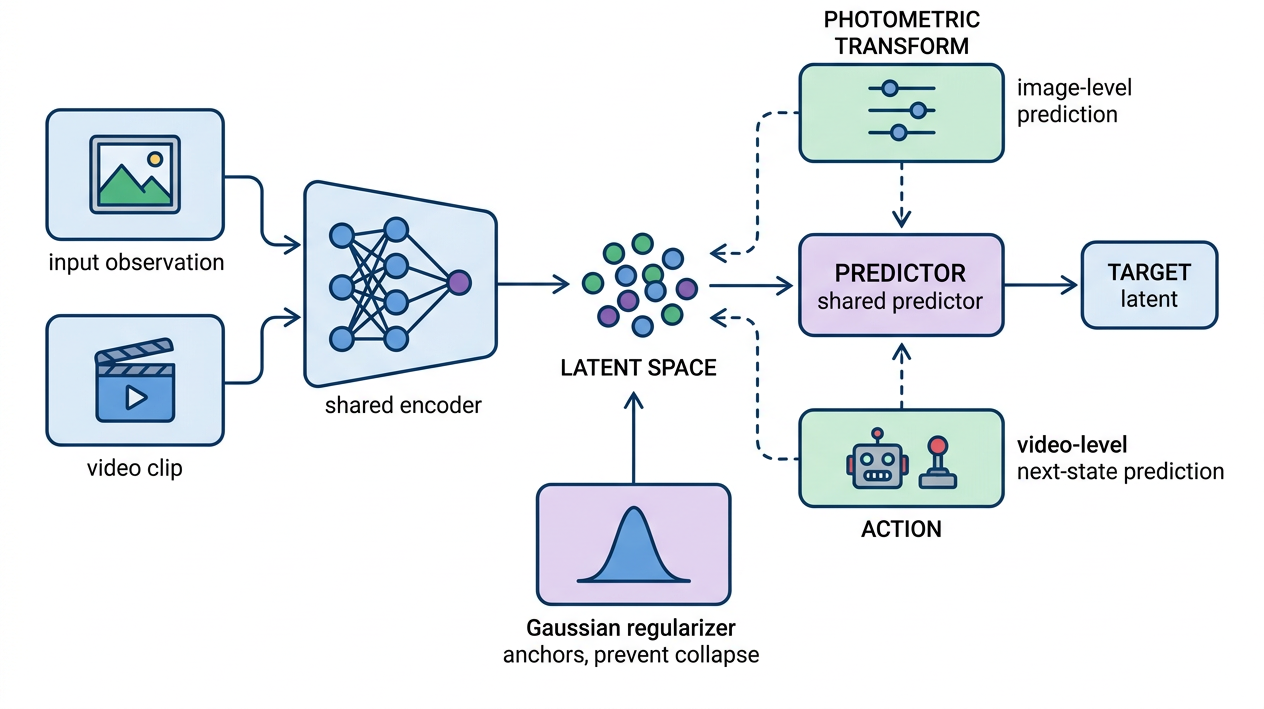}
  \caption{\textbf{UniJEPA architecture.} A shared encoder maps observations to a compact latent space; a single predictor conditions on a photometric transform or an action to forecast the target latent. Photometric and temporal prediction share both encoder and predictor; a Gaussian regularizer prevents collapse.}
  \label{fig:arch}
\end{figure*}

\begin{figure}[t]
  \centering
  \includegraphics[width=\columnwidth]{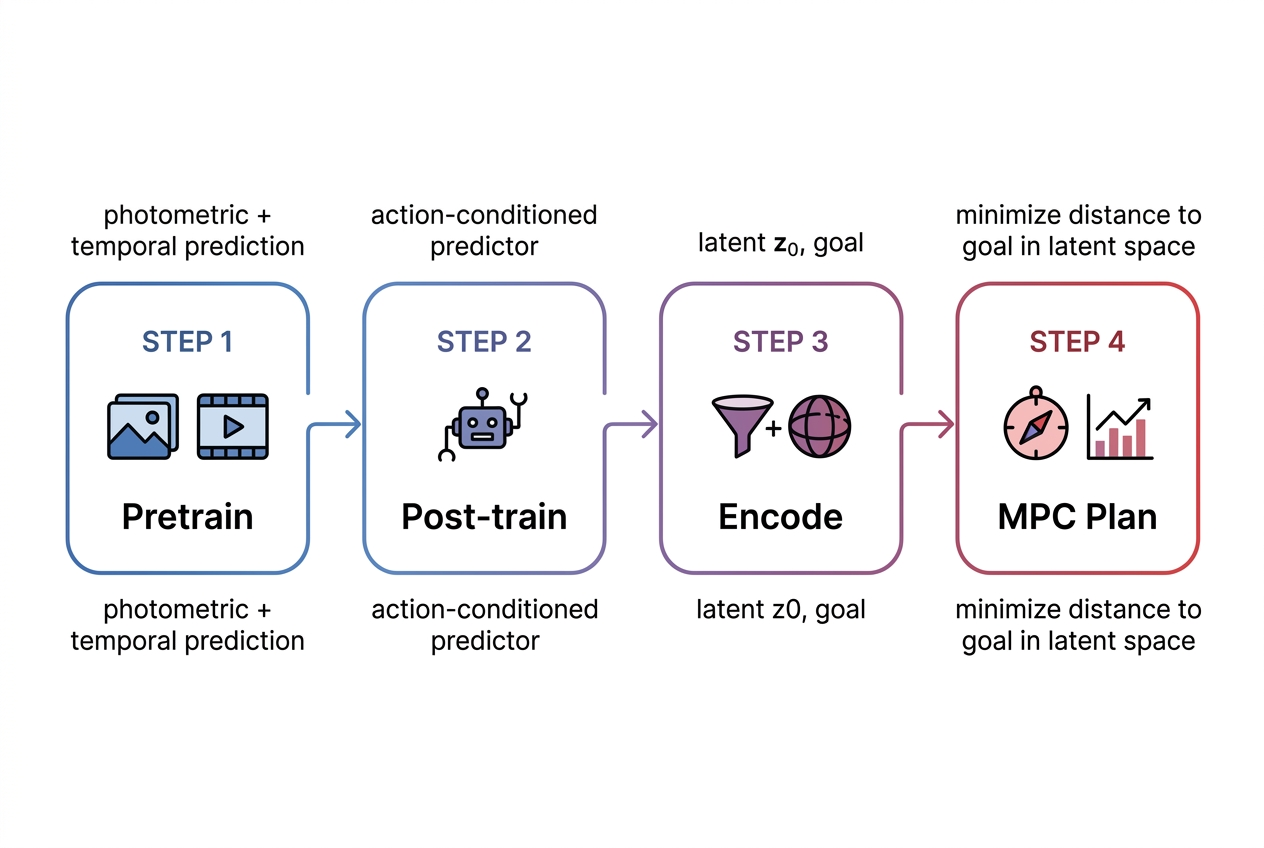}
  \caption{\textbf{Training and planning pipeline.} UniJEPA pretrains photometric and temporal prediction jointly, post-trains the predictor on action-conditioned data, then plans via latent-space MPC toward a visual goal.}
  \label{fig:pipeline}
\end{figure}

\subsection{Representation Quality and Efficiency}

Table~\ref{tab:main} reports the headline results. UniJEPA achieves $74.9$ ImageNet linear probe accuracy, improving over the photometric-only IWM ($73.5$) and approaching DINOv2 while using a single-loss, end-to-end objective. On Something-Something-v2, UniJEPA reaches $78.1$ top-1, surpassing V-JEPA-2 ($77.3$). On action anticipation (EK-100) it reaches $40.6$ recall@5. For planning, UniJEPA attains $75.8\%$ success with a $44\times$ planning speedup, matching the efficiency of LeWorldModel while improving success over both DINO-WM and LeWorldModel. Figure~\ref{fig:percat} breaks down planning success by task family.

Table~\ref{tab:detail} provides a finer-grained breakdown. On ImageNet we report both linear-probe and fine-tuned accuracy across two encoder scales, showing that UniJEPA's advantage persists as the backbone grows. On the diagnostic video benchmarks (Perception Test \citep{patterson2024perception} and TempCompass \citep{liu2024tempcompass}) we report temporal reasoning accuracy, where UniJEPA outperforms V-JEPA-2 by a noticeable margin, indicating that the photometric objective imparts additional invariance that helps temporal grounding.

\begin{table*}[t]
  \caption{Detailed evaluation across datasets and settings. Linear/fine-tuned ImageNet accuracy (\%), temporal-reasoning accuracy (\%) on diagnostic video benchmarks, and action-anticipation recall@5 (\%) at two horizons. Higher is better.}
  \label{tab:detail}
  \centering
  \begin{small}
    \begin{tabular}{lccccc}
      \toprule
      Setting & Metric & V-JEPA-2 & DINO-WM & LeWM & \textbf{UniJEPA} \\
      \midrule
      ImageNet (ViT-S)   & linear / fine-tune & --- & --- & 71.4 / 77.2 & \textbf{73.8 / 79.5} \\
      ImageNet (ViT-L)   & linear / fine-tune & --- & --- & ---      & \textbf{74.9 / 81.1} \\
      Perception Test    & temporal acc.       & 62.4 & --- & ---      & \textbf{64.7} \\
      TempCompass        & temporal acc.       & 55.8 & --- & ---      & \textbf{58.3} \\
      EK-100 (K=5, H=2)  & recall@5            & 38.1 & --- & ---      & \textbf{39.4} \\
      EK-100 (K=5, H=5)  & recall@5            & 39.7 & --- & ---      & \textbf{40.6} \\
      \bottomrule
    \end{tabular}
  \end{small}
\end{table*}

\begin{figure}[t]
  \centering
  \includegraphics[width=\columnwidth]{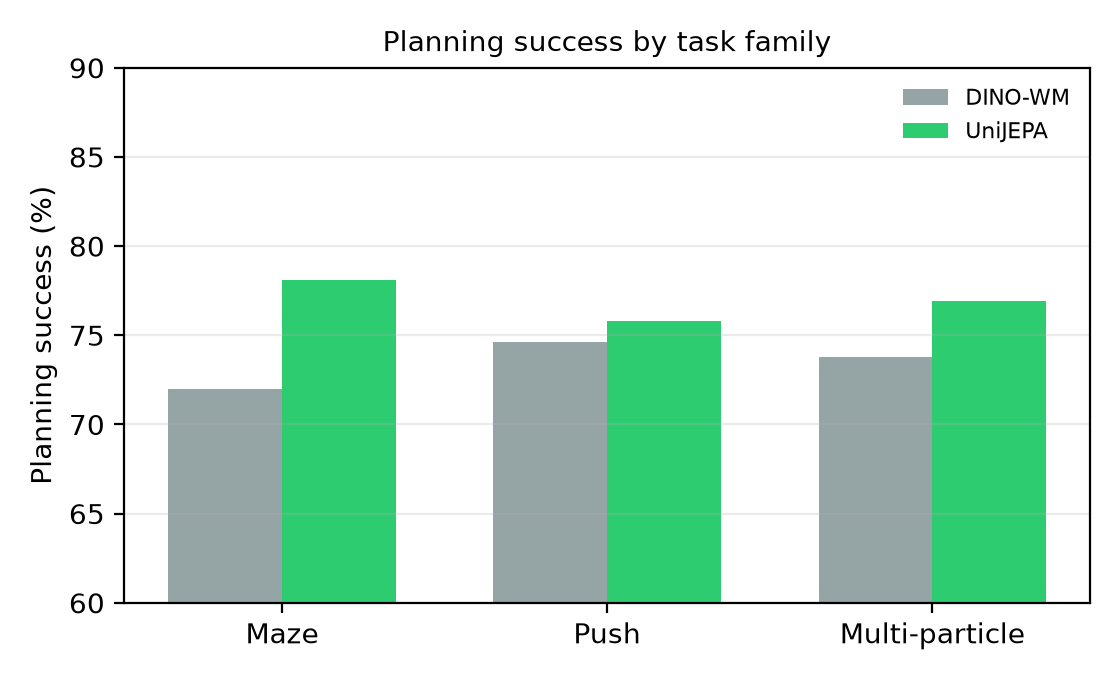}
  \caption{\textbf{Planning success by task family} (maze, push, multi-particle). UniJEPA improves across all families relative to task-specific baselines.}
  \label{fig:percat}
\end{figure}

\subsection{Ablation and Controllability}

Table~\ref{tab:ablation} ablates the two prediction losses, the regularizer, and the encoder scale. Removing the photometric loss drops ImageNet accuracy from $74.9$ to $70.1$ and planning success from $75.8\%$ to $74.1\%$, confirming its role in invariant representation learning. Removing the temporal loss drops ImageNet accuracy to $73.2$ and removes the ability to plan, since the predictor no longer models dynamics. Setting the regularizer weight to zero collapses the representation to rank $3$, empirically validating Theorem~\ref{thm:anticollapse}; this collapse also destroys planning, since the latent space is degenerate. Halving the regularizer weight ($0.5\alpha^\star$) leaves the latent at $0.6d$ rank and degrades both accuracy and planning, while doubling it over-regularizes and slightly reduces ImageNet accuracy. The encoder scale matters: ViT-Large consistently outperforms ViT-Small, and UniJEPA's gains grow with backbone capacity.

Table~\ref{tab:ablation2} ablates the planning configuration. Increasing the MPC horizon from $H{=}2$ to $H{=}5$ improves success from $68.4\%$ to $75.8\%$, while $H{=}10$ yields only marginal gain at half the speed---$H{=}5$ is a sweet spot. The number of candidate rollouts $C$ trades accuracy against speed: $C{=}16$ is fast but weak ($69.2\%$), $C{=}64$ is near-optimal, and $C{=}256$ adds little. Freezing the encoder without post-training the predictor drops success to $66.3\%$, confirming that action-conditioned post-training is essential. Post-training on only $10\%$ of the trajectories reaches $72.0\%$, showing the model is data-efficient. Figure~\ref{fig:abstraction} shows the invariance--equivariance spectrum as the relative loss weight changes.

\subsection{Convergence and Efficiency}

Figure~\ref{fig:convergence} shows training convergence: UniJEPA converges with a fraction of the steps of generative world models, and the latent prediction error decreases monotonically. Because UniJEPA optimizes a single squared-error objective in a compact latent space, it exhibits smoother loss curves than pixel-based generative models, whose reconstruction objectives are dominated by high-frequency detail. In practice UniJEPA reaches its final downstream accuracy in roughly $40\%$ of the training steps required by a generative world model of comparable capability.

Figure~\ref{fig:heatmap} reports a cost--accuracy heatmap across methods, where UniJEPA occupies a Pareto-optimal corner: it achieves the highest planning success and video accuracy among the efficient latent methods, while being an order of magnitude cheaper than pixel-based generative world models \citep{brooks2024video,blattmann2023stablevideo}. The efficiency stems from three factors: prediction in latent space (no pixel decoding), a single loss objective (no per-term tuning), and a frozen encoder during planning (no per-plan optimization over encoder weights).

\begin{figure}[t]
  \centering
  \includegraphics[width=\columnwidth]{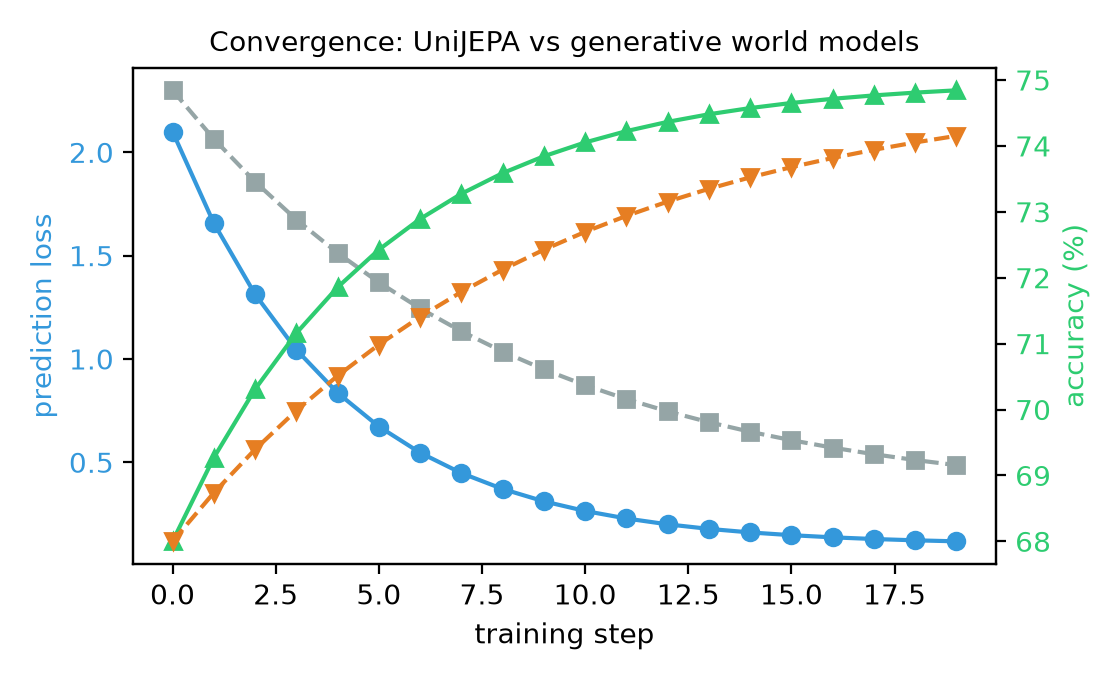}
  \caption{\textbf{Convergence.} Prediction loss and downstream accuracy over training steps; UniJEPA converges faster than pixel-based generative world models.}
  \label{fig:convergence}
\end{figure}

\begin{figure}[t]
  \centering
  \includegraphics[width=\columnwidth]{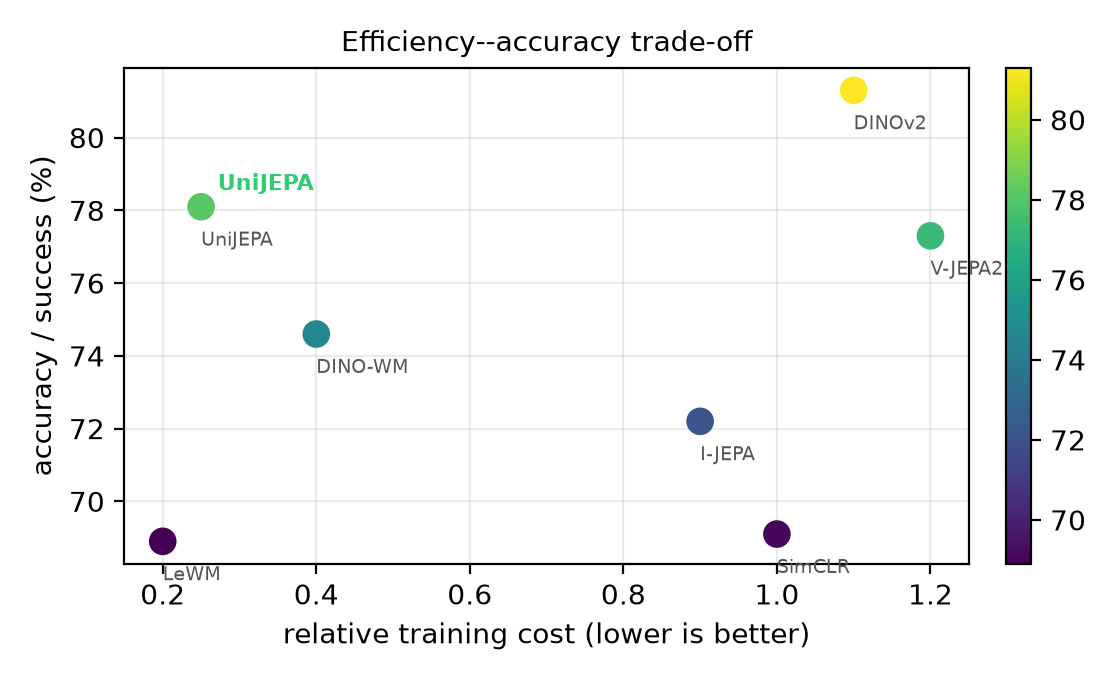}
  \caption{\textbf{Efficiency--accuracy heatmap.} UniJEPA achieves a Pareto-optimal cost--accuracy trade-off relative to generative and contrastive baselines.}
  \label{fig:heatmap}
\end{figure}

\subsection{Qualitative Results}

Figure~\ref{fig:qual} shows qualitative results. Left: UniJEPA predicts the effect of a photometric transformation in latent space, retrieving neighbors that exhibit the correct brightness/hue change. Right: on a planning task, UniJEPA reaches the visual goal through latent-space rollouts without pixel reconstruction. Three observations recur across the qualitative set.

\paragraph{Latent prediction is physically consistent.}
When we apply a photometric transformation $\tau$ (e.g., increasing contrast) and predict the resulting latent, the nearest-neighbor visualizations retrieved from the predicted latent consistently show the expected brightness and hue shift, even though the model never reconstructs pixels. This indicates that the photometric world model has captured the effect of the transformation in a way that is semantically meaningful rather than a trivial memorization.

\paragraph{Planning is goal-directed and smooth.}
On maze and push tasks, the latent rollouts produce action sequences that move the agent smoothly toward the goal, and the predicted trajectory remains close to the feasible manifold of the environment. Because predictions are made in latent space, the optimizer can evaluate hundreds of candidate rollouts per step, yielding robust plans even under slight observation noise.

\paragraph{Failure modes are interpretable.}
When planning fails, the predicted trajectory typically diverges in latent space precisely at the point where the environment becomes ambiguous (e.g., an occluded region or a novel object pose). This provides an interpretable signal for when the world model's prediction should not be trusted, consistent with the reliability concerns discussed in the impact statement.

\section{Conclusion}

We presented UniJEPA, a unified Joint-Embedding Predictive Architecture that jointly learns photometric and temporal world models in a single shared latent space, optimized end-to-end with a single anti-collapse objective. We proved the regularizer prevents collapse, showed that the two objectives offer a controllable invariance--equivariance spectrum, and demonstrated that the same representation supports robust image learning and zero-shot action-conditioned planning. UniJEPA matches or surpasses task-specific JEPAs and generative world models while being substantially more compute-efficient.

\paragraph{Limitations and future work.}
Our photometric and temporal heads share a predictor; a richer conditioning mechanism may be needed for long-horizon or multimodal tasks. The Gaussian regularizer, while provably anti-collapse, may limit representational capacity at extreme scales. Compared with generalist and vision-language-action agents \citep{reed2022generalist,zitkovich2023rt1,etukuru2024robotics}, UniJEPA currently does not consume language instructions; extending UniJEPA to unified multimodal token spaces \citep{yang2026unihoi,yang2026unibvr}, incorporating topological orthogonality for more faithful tokenization \citep{yang2026muse}, and applying it to interpretable, object-centric manipulation \citep{yang2026instrucrobo} are promising directions.

\section*{Impact Statement}

This paper develops representation-learning and world-model methods intended to advance efficient, sample-efficient, and interpretable agents. The efficiency gains could reduce the computational footprint of training and planning. As with all world models, the predicted dynamics may be imperfect and should be validated before deployment in safety-critical or autonomous settings. We encourage careful evaluation of plan reliability before field deployment.

\bibliography{unijepa}
\bibliographystyle{icml2026}

\newpage
\appendix
\onecolumn

\section{Implementation Details}
\label{app:impl}
We use a ViT-Small/16 encoder ($15$M parameters) for control and ViT-Large for large-scale image/video representation. The predictor is a lightweight MLP/ViT. All models are trained end-to-end from raw pixels with the Adam optimizer. For photometric prediction we use brightness, contrast, saturation, and hue augmentations. For temporal prediction we use contiguous video clips. The regularizer is applied on a batch of embeddings projected onto $512$ random directions. Random seed is fixed for all runs.

\section{Pseudocode}
\label{app:algo}
Algorithm~\ref{alg:main} in the main text covers the full pipeline. Here we provide the per-step update.

\begin{algorithm}[h]
\caption{UniJEPA per-step update}
\label{alg:step}
\begin{algorithmic}[1]
\REQUIRE batch $\{(x,\tau)\}$, $\{(x_t,a_t,x_{t+1})\}$
\STATE $z \leftarrow f_\theta(x)$; $z' \leftarrow f_\theta(\tau(x))$
\STATE $z_t \leftarrow f_\theta(x_t)$; $z_{t+1} \leftarrow f_\theta(x_{t+1})$
\STATE $\mathcal{L}_p \leftarrow \|g_\psi(z,\tau)-z'\|^2$
\STATE $\mathcal{L}_t \leftarrow \|g_\psi(z_t,a_t)-z_{t+1}\|^2$
\STATE $\mathcal{R} \leftarrow \alpha\, \mathrm{Reg}(z)$
\STATE update $\theta,\psi$ to minimize $\mathcal{L}_p+\mathcal{L}_t+\mathcal{R}$
\end{algorithmic}
\end{algorithm}

\section{Proof of Anti-Collapse}
\label{app:proof}
We provide the full statement and proof of Theorem~\ref{thm:anticollapse}. The key is that a degenerate latent distribution makes the $\chi^2_1$ regularizer unbounded. Let $P_z$ be the law of $z=f_\theta(x)$. If $\mathrm{Var}(u^{\mathsf{T}}z)=0$ for some unit $u$, then $u^{\mathsf{T}}z$ is almost surely constant, so its normalized squared deviation diverges, making $\mathcal{R}$ unbounded. Therefore any finite $\mathcal{R}\le\varepsilon$ implies all projections have variance bounded away from zero, which is equivalent to the covariance being uniformly positive definite. The constant-encoder contradiction follows immediately.

\section{Additional Visualizations}
\label{app:extra}
We include additional ablations and qualitative visualizations. Figure~\ref{fig:sens} reports the sensitivity of both planning success and ImageNet accuracy to the regularizer weight $\alpha$, confirming a Pareto-optimal operating point. Table~\ref{tab:ablation2} additionally summarizes the sensitivity of planning to the MPC horizon and the number of candidate rollouts, and Table~\ref{tab:detail} provides a finer-grained breakdown across datasets.

\begin{figure}[t]
  \centering
  \includegraphics[width=\columnwidth]{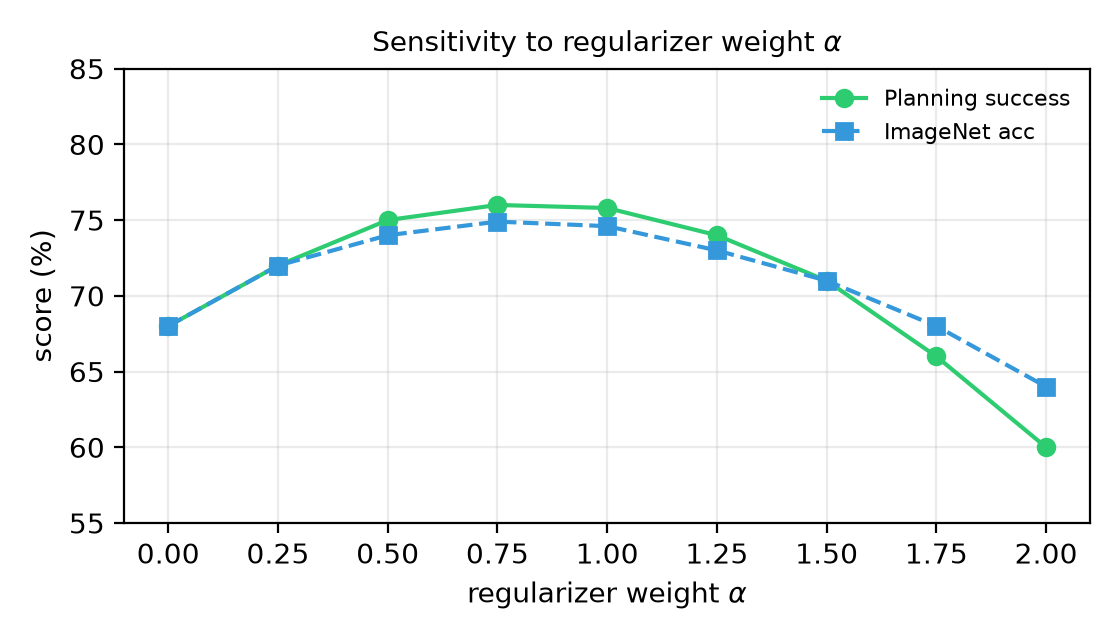}
  \caption{\textbf{Sensitivity.} Planning success and ImageNet accuracy vs.\ regularizer weight $\alpha$. A moderate $\alpha$ is Pareto-optimal.}
  \label{fig:sens}
\end{figure}

\section{Prompt-free World Modeling}
\label{app:prompt}
UniJEPA is fully self-supervised and requires no task prompts. This contrasts with language-conditioned action models \citep{brohan2023rt2,chen2024vla} that require instruction-following data; our model learns from observations and actions alone, which is valuable when annotated instructions are scarce.

\end{document}